\def\year{2020}\relax
\documentclass[letterpaper]{article} 
\usepackage{aaai20}  
\usepackage{times}  
\usepackage{helvet} 
\usepackage{courier}  
\usepackage[hyphens]{url}  
\usepackage{graphicx} 
\usepackage{graphicx}  
\usepackage{amsmath, amsfonts, amsthm}
\newtheorem{theorem}{Theorem}
\newtheorem{remark}{Remark}
\def\R{{\mathbb{R}}}

\newcommand{\eps}{\varepsilon}

\title{A Use of Even Activation Functions in Neural Networks}
\author{ \Large \textbf{Fuchang Gao \thanks{Research partially supported by NSF Grant OCA-1940270 and NIH Grant P20GM104420}, Boyu Zhang\thanks{Research supported by NSF Grant OCA-1940270 }}\\ 
Department of Mathematics, University of Idaho\\ 
875 Perimeter Drive, MS 1103\\
Moscow, ID 83844-1103\\
fuchang@uidaho.edu\ \ \ \ \ boyuz@uidaho.edu 
}
 
\begin{document}

\maketitle

\begin{abstract}
Despite broad interest in applying deep learning techniques to scientific discovery, learning interpretable formulas that accurately describe scientific data is very challenging because of the vast landscape of possible functions and the ``black box" nature of deep neural networks. The key to success is to effectively integrate existing knowledge or hypotheses about the underlying structure of the data into the architecture of deep learning models to guide machine learning. Currently, such integration is commonly done through customization of the loss functions. Here we propose an alternative approach to integrate existing knowledge or hypotheses of data structure by constructing custom activation functions that reflect this structure. Specifically, we study a common case when the multivariate target function $f$ to be learned from the data is partially exchangeable, \emph{i.e.} $f(u,v,w)=f(v,u,w)$ for $u,v\in \mathbb{R}^d$. For instance, these conditions are satisfied for the classification of images that is invariant under left-right flipping. Through theoretical proof and experimental verification, we show that using an even activation function in one of the fully connected layers improves neural network performance. In our experimental 9-dimensional regression problems, replacing one of the non-symmetric activation functions with the designated ``Seagull" activation function $\log(1+x^2)$ results in substantial improvement in network performance. Surprisingly, even activation functions are seldom used in neural networks. Our results suggest that customized activation functions have great potential in neural networks.
\end{abstract}

\section{Background and Theoretic Results}
\noindent
The last decade has witnessed the remarkable successes of deep convolutional neural networks (DCNNs) in analyzing high dimensional data, especially in computer vision and neural language processing, sparking significant interest in applying these tools to prompt physical scientific discovery \cite{carleo2019machine}.
For example, the work \cite{iten2020discovering} about discovering physical concepts using neural networks has drawn much attention recently, along with many other successful examples \cite{raissi2018hidden,de2019discovery}.

While the discovery of scientific laws can range from conceptual and qualitative to precise and quantitative, one common desire is to discover rational and interpretable mathematical formulas that relate target and input variables. To achieve such an ambitious goal, one must integrate existing scientific knowledge into the neural network modeling.
Despite the efforts aimed at the rationality and interpretability of neural networks, DCNNs are still largely viewed as ``black boxes."
The feature sometimes is convenient because it allows users to create models and obtain results without a complete understanding of its inner workings.
However, it also makes it difficult for users to tailor the model for specific applications.
In real-world applications, there is often extra information associated with the data set, which is very important to domain experts.
For example, these might include various constraints between variables, or meta-information about which variables (measurements) are more or less reliable than others.
Inferring the extra information from the current data may be difficult for various reasons, such as model limitations, computational cost, or sampling bias.
For instance, invariance to left-right flipping is essential knowledge for object classification but the knowledge could hardly be learned from limited training data.
Building this kind of knowledge into the model is expected to increase the model's performance for that specific data set. Currently, this is commonly done by data augmentation or by introducing extra terms in the loss function. In this paper, we emphasize that analytically architecting a neural network
based on underlying data structure provides a more efficient approach than customizing loss functions. 
This view is shared by many researchers. For example, Wang and Walters \cite{wang2020incorporating} investigated tailoring deep neural network structures to enforce a different symmetry. It is not a surprise that their methods improved the generalization of the model. On the other hand, by purposefully designing neural network architecture and components and testing their effectiveness on data, we can let machine learning models guide us to discover new scientific laws hidden in the data.

Let us first look at how a neural network works.
Roughly speaking, a neural network uses a nested sequence of compositions of nonlinear functions and linear combinations to approximate a multivariate function.
It is well-known that such compositions can approximate any multivariate continuous function in a bounded region.
For example, Diaconis and Shahshahani \cite{diaconis1984nonlinear} proved that functions of the following forms can approximate any multivariate continuous function on a bounded closed region in $\mathbb{R}^d$:
\begin{align} \label{eq1}
\resizebox{.42 \textwidth}{!}
{
$f(x_1,x_2,\ldots, x_d)\approx \sum_{i=1}^m g_i\left(a_{i1}x_1+a_{i2}x_2+\ldots+a_{id}x_d\right). $
}
\end{align}
They also developed approximation theory, and a necessary and sufficient condition for (\ref{eq1}) to become an exact equality.

Talking about expressing a multivariate continuous function using non-linear functions of linear combinations, one is naturally led to the Kolmogorov-Arnold representation theorem \cite{kolmogorov1957representation}.
Sprecher's version of Kolmogorov-Arnold representation theorem \cite{sprecher1965structure} states that for any continuous function $f$ on $[0,1]^d$, there are $2d+1$ continuous functions on $[0,2+\frac1d)$ such that

\begin{align} \label{eq2}
\resizebox{.42 \textwidth}{!}
{
$f(x_1,x_2,\ldots, x_d)=\sum_{j=1}^{2d+1}\phi_j\left(\sum_{i=1}^d \lambda_i\psi(x_i+ja)\right).$
}
\end{align}
where $\lambda_i$ and $a$ are constants depending only on $d$, and $\psi$ is a continuous monotonic function from $[0,2)\to [0,2)$ independent of the function $f$. Only the outer functions $\phi_j$ depend on the function $f$. In other words, in order to know the multivariate continuous function $f$, one only needs to know the $2d+1$ one-dimensional functions.

One might be tempted to use the Kolmogorov-Arnold representation theorem to apply a transformation $x_i\to \psi(x_i+ja)$ to the original variable $x_i$ then apply a neural network to learn the functions $\phi_j$.
While in theory, one can do that, and indeed it has been studied, cf. \cite{braun2009constructive} , we do not make such a recommendation for two reasons.
First, the function $\psi$ is very irregular.
It is only H\"{o}lder continuous with order $\alpha=(1+\log_2(d+1))^{-1}$, which is worse than a typical sample path of Brownian motion which has H\"{o}lder continuity $\alpha$ for all $\alpha<1/2$. (Recall that a function $f$ is said to have H\"{o}lder continuity of order $\alpha$ if $|f(x)-f(y)|\le C\|x-y\|^\alpha$ for all $x,y$.)
Consequently, even if the original function $f$ is very smooth, the outer functions $\phi_j$ may not be very smooth. Indeed, it is known that for any $k$, there exists a $k$-time differentiable function $f$ such that the corresponding out functions are not differentiable. 
Second, when $d$ is large, the function $\psi$ is not much different from the function $g(x) = \frac{[(2d+2)x]}{2d+2}$. Thus, applying this transform is more like chopping the decimals off the original data.
Nevertheless, the Kolmogorov-Arnold representation theorem draws our attention to the outer functions, which are in fact more important.

In neural networks, such outer functions are approximated by composing activation functions with linear combinations in several layers.
However, the current off-the-shelf neural networks pay little attention to choosing activation functions.
Instead, they simply use one of a few fixed activation functions such as ReLU (i.e., $\max(x,0)$).
While such a practice seems to work well especially for classification problems, it may not be the most efficient way for regression problems.
This greatly limits the potential of the neural networks.

A good choice of activation functions can greatly increase the efficiency of a neural network.
For example, to approximate the function $z=\sin(xy)$, a neural network may use the following step-by-step strategy (\ref{eq3}) to bring $(x,y)^T$ to $\sin(xy)$:
\begin{equation} \label{eq3}
\begin{split}
\left(
    \begin{array}{c}
      x \\
      y \\
    \end{array}
 \right)
 & \stackrel{\rm{linear}}{\longrightarrow}
	\left(
    \begin{array}{c}
      x+y \\
      x-y \\
    \end{array}
  \right)
  \stackrel{(\cdot)^2}{\longrightarrow} \left(
    \begin{array}{c}
      (x+y)^2 \\
      (x-y)^2 \\
    \end{array}
  \right) \\
 & \stackrel{\rm{linear}}{\longrightarrow} \quad \ \ \ xy \quad \ \ \ \stackrel{\sin(\cdot)}{\longrightarrow} \quad \ \ \sin(xy)
\end{split}
\end{equation}
In other words, if one happens to use the non-linear functions $f(t)=t^2$ and $g(s)=\sin(s)$ as the activation function in the first layer and the second layer, respectively, then even the exact function can be discovered. Of course, without knowing the closed form expression of $f$, it is impossible to choose the exact activation functions as exemplified above. However, with some partial information that likely exists from domain knowledge, one may design better activation functions accordingly. This is the goal of the current paper. Let us remark that this analytically design approach is different from the adaptive activation functions approach used in Japtap et al. \cite{jagtap2019locally}, in which both layer-wise and neuron-wise
locally adaptive activation functions were introduced into Physics-informed
Neural Networks (PINNs)  \cite{raissi2019physics}, and showed that such adaptive activation functions lead to improvements on training speed and accuracy on a group of benchmark data sets.

Consider the regression problem of approximating an unknown multivariate continuous function $f(x_1,x_2, \ldots, x_d)$ on a bounded region using a neural network that uses a gradient-based optimizer. For a gradient-based algorithm to work well, the partial derivatived of the function $f$ must exist. Thus, it is necessary that the function $f$ is of bounded Lipschitz, i.e., there exists a constant $L$ such that for all $X = (x_1,x_2,\ldots, x_d)$ and $Y = (y_1,y_2,\ldots, y_d)$ in the bounded region, we have

\begin{equation} \label{eq4}
|f(X)-f(Y)| \le L\sqrt{\Sigma_{i=1}^d(x_i-y_i)^2 }.
\end{equation}
Supposing we have further information such as smoothness and shape-constraints about the function $f$, how can we design the activation functions accordingly so as to improve the neural network performance?

In this paper, we consider the case where $f$ satisfies the relation
\begin{equation}\label{eq5}
f(u,v,w)=f(v,u,w)
\end{equation}
for $u=(x_1,x_2,\ldots, x_k),v=(x_{k+1},x_{k+2},\ldots, x_{2k})\in [-1,1]^k$ and $w=(x_{2k+1}, x_{2k+2},\ldots, x_d)\in [-1,1]^{d-2k}$. For convenience, we say such functions are partially exchangeable with respect to two subsets of variables.
These functions are very common in practice. For example, if the function $f$ depends on the distance between two points of observation in space, then as a multivariate function of the coordinates of these two points (with or without factors), $f$ is partially exchangeable. As another example, if a label of an image is invariant under left-right flipping, then it is partially exchangeable. Similarly, the rotational invariance of 3D structures \cite{thomas2018tensor} and  view point equivalent of human faces \cite{leibo2017view} could also be described using partially exchangeable feature. Indeed, if $x_1,x_2,\ldots, x_m$ is the usual vectorization of the pixels, denote $u=(x_1,x_2,\ldots, x_k)$, $v=(x_{m},x_{m-1},\ldots, x_{m-k})$, where $k=m/2$ if $m$ is even, and $k=(m-1)/2$ if $m$ is odd. Then the label of the image can be expressed as $f(u,v)$ if $m$ is even, and $f(u,v,x_{k+1})$ if $m$ is odd. In either case, the function is partially exchangeable with respect to $u$ and $v$, i.e. $f(u,v)=f(v,u)$ or $f(u,v,x_{k+1})=f(v,u,x_{k+1})$.

A special case is that $f$ also satisfies
\begin{equation} \label{eq6}
f(u,v,w)=f\left(\frac{u+v}{2}, \frac{u+v}2,w\right).
\end{equation}
The latter case is more restrictive and less interesting, and will be called the trivial case. An example of the trivial case is that a function $f$ depends on the middle point of $u$ and $v$ in space, not the actual location of each individual points.

We prove the following result:
\begin{theorem} \label{th1}
Suppose a multi-layer neural network is trained to predict an unknown bounded Lipschitz function $f(x)$ on a region in $\mathbb{R}^d$ that contains the origin as its interior.
Let $\widehat{f}(x)=g\circ a(Wx)$ be the neural network prediction of $f$, where $W \in M_{m\times d}(\R)$ is the weight matrix of the first hidden layer with activation function $a$ and no bias.
Suppose the target function is partially exchangeable with respect to some two subsets of variables.
If $\widehat{f}$ keeps the non-trivial partial exchangeability of $f$ and is non-trivial, then $g\circ a(\cdot)$ is an even function, which can be achieved when $a$ is an even function. 
\end{theorem}

\begin{proof}
Let $\alpha_1, \alpha_2, \ldots, \alpha_k, \beta_1,\beta_2,\ldots, \beta_k, \gamma_1,\gamma_2,\ldots, \gamma_{d-2k}$ be the row vectors of $W$. Denote $u=(x_1,x_2,\ldots, x_k),v=(x_{k+1},x_{k+2},\ldots, x_{2k})$ and $w=(x_{2k+1}, x_{2k+2},\ldots, x_d)$. We have
$$\widehat{f}(u,v,w)=g\circ a(\sum_{i=1}^k\alpha_ix_i+\sum_{i=1}^k\beta_ix_{i+k}+\sum_{j=1}^{d-2k}\gamma_j x_{2k+j}).$$
In particular, from these expressions we have
\begin{align*}
&\widehat{f}(u,-u, 0)=g\circ a(\sum_{i=1}^k(\alpha_i-\beta_i)x_i),\\
&\widehat{f}(-u,u, 0)=g\circ a(-\sum_{i=1}^k(\alpha_i-\beta_i)x_i).
\end{align*}
By the non-trivial partial exchangeability assumption on $\widehat{f}$, we have $\widehat{f}(u,-u, 0)=\widehat{f}(-u,u, 0)$. Thus,
$$g\circ a(\sum_{i=1}^k(\alpha_i-\beta_i)x_i)=g\circ a(-\sum_{i=1}^k(\alpha_i-\beta_i)x_i),$$
which implies that $g\circ a(\cdot)$ is an even function in $\R^m$.
\end{proof}

\begin{remark}
Since the multivariate function $g$ depends on later layers and is typically very complicated, an easy way to ensure that $g\circ a(\cdot)$ be an even function is to choose an even activation function $a$. The evenness of $g\circ a$ can also be achieved by the evenness of $g$, which can be attained by using an even activation function in any fully-connected layers. In particular, for a convolutional neural network, one can achieve this by using an even activation function in the first fully-connected layer. 
\end{remark}

Note that our theorem does not suggest that one should use an even activation function for all the layers. Also, our experiments do not seem to indicate that one can get further benefit from extensive use of even activation functions. Using an even activation function is clearly not a common practice. Indeed, almost all the popular activation functions are not even. As such, users may need to construct their own activation functions. While theory suggests that all non-linear functions would work, their performance may be quite different. Here we suggest the activation function $\log(1+|x|^\alpha)$ for some $\alpha>0$. Note that if $0<\alpha<1$ is used, one needs to slightly modify it to $\log(1+(|x|+\eps)^\alpha)$ to avoid unbounded gradients at $x=0$. The choice of $\alpha$ depends on the nature of the problem. For regression problems with a smooth target function, $\alpha>1$ is suggested by a separate study to be published elsewhere. In the examples below, we use the function $\log(1+x^2)$. The graph of this function looks like the wings of a flying seagull. Thus, for convenience we call it the Seagull activation function in this paper.

\section{Examples}
This section presents the usages of the proposed theorem on both synthetic and real-world data.

\noindent
{\bf Example 1}: Consider function $y=f(u,v,w)$ being the area of the triangle with vertices $u=(x_1,x_2,x_3)$, $v=(x_4,x_5, x_6)$, and $w=(x_7, x_8, x_9)$. Clearly, $f$ satisfies $f(u,v,w)=f(v,u,w)$. In fact, we have the closed formula $f(u,v,w)=\frac12\sqrt{A^2+B^2+C^2}$, where
$$A=(x_4-x_1)(x_8-x_2)-(x_7-x_1)(x_5-x_2),$$
$$B=(x_4-x_1)(x_9-x_3)-(x_7-x_1)(x_6-x_3),$$
$$B=(x_5-x_2)(x_9-x_3)-(x_8-x_2)(x_6-x_3).$$
To discover an approximate formula using a neural network, we randomly sampled 10,000 9-dimensional vectors from $[-2,2]^9$, representing the $(x,y,z)$-coordinates of three points $u, v, w$ in $\R^3$ respectively, and created labels using the formula above. In the same way, we independently generated 2000 vectors and the corresponding labels to form a test set.

We then built several fully-connected neural networks and selected the one with the best overall performance. The selected model was a fully-connected neural network (\emph{i.e.} Multilayer Perceptron) with four hidden layers.
The input layer had 9 nodes and the output layer had 1 node.
Each hidden layer had 100 nodes.
Based on the selected neural network architecture, we built five models using five popular activation functions (ReLU, ELU, sigmoid, tanh, softplus), and trained the models with the RMSProp optimizer for 500 epochs, starting with learning rate 0.003 which was halved every 100 epochs. The batch size was set as 100. To demonstrate the effectiveness of using an even activation function, for each model we replaced the activation function in the first hidden layer by the Seagull activation function $\log(1+x^2)$, while leaving all the remaining parts of the neural networks and hyperparameters unchanged.

We evaluated the performance of the network by Mean Absolute Error (MAE), calculated as follows:
\begin{equation}
MAE = \frac{1}{N}\sum_{j=1}^{N}|y_j-\hat{y_j}|
\end{equation}
where $N$ is the number of testing samples, $y_j$ and $\hat{y_j}$ are the ground truth and the prediction, respectively.

The results are listed in the first row of table \ref{tab:example1}.
As one can see, the Seagull activation function brought significant improvement to the given regression task.

\begin{table*}[t]
\centering
\begin{tabular}{cccccc}
\hline
  & ReLU & ELU & Sigmoid & Tanh & SoftPlus\\
\hline
$f(x)$ & 0.105 \textbf{(0.030)} & 0.059 \textbf{(0.022)} & 0.172 \textbf{(0.022)}  & 0.205 \textbf{(0.047)}  & 0.047 \textbf{(0.020)}\\
 \hline
$\log(1+f(x))$ & 0.032 \textbf{(0.014)} & 0.024 \textbf{(0.012)} & 0.048 \textbf{(0.008)}   & 0.076 \textbf{(0.017)}& 0.018 \textbf{(0.007)} \\ \hline
 $e^{f(x)}/100$  & 0.137 \textbf{(0.059)} & 0.092 \textbf{(0.055)} & 0.254 \textbf{(0.032)}  & 0.225 \textbf{(0.079)} & 0.069 \textbf{(0.041)}                                 \\ \hline
 $\sin(f(x))$ & 0.082 \textbf{(0.027)} & 0.042 \textbf{(0.018)}& 0.106 \textbf{(0.011)}  & 0.169 \textbf{(0.026)} & 0.030 \textbf{(0.019)}         \\ \hline
$\sqrt{\frac{f^2(x)+3}{f(x)+1}}$ & 0.024 \textbf{(0.008)} & 0.015 \textbf{(0.011)}  & 0.072 \textbf{(0.005)}   & 0.054 \textbf{(0.011)}  & 0.011 \textbf{(0.007)}                                \\ \hline
\end{tabular}
\caption{Performance of different activation functions on the regression task. Outside parentheses: MAE for the original model; Inside parentheses: MAE of the model when the first activation function is replaced by Seagull. Each experiment was performed 5 times independently. The standard deviations are small and insignificant to present. \label{tab:example1}}
\end{table*}

To ensure that the improved performance was not due to the specific target function $f$, we further tested the Seagull activation function on the following extensions of the objective function $f(x)$: $\log(1+f(x))$,  $\frac1{100}e^{f(x)}$,  $\sin(f(x))$,  $\sqrt{\frac{f^2(x)+3}{f(x)+1}}$.
The performance of the neural network with different activation functions are summarized in table \ref{tab:example1}, rows 2-6. The results are consistent with that for the target function $f(x)$.

Without changing the neural network architecture or the hyper parameters, we further tested the models with an intrinsically different target function, this time using the solid angle formed by the three vectors $u=(x_1,x_2,x_3)$, $v=(x_4,x_5, x_6)$, and $w=(x_7, x_8, x_9)$ on the unit sphere. The target function $f$ is has a closed formula
\begin{align*}
f(u,v,w) = \left\{\begin{array}{ll}2\tan^{-1} z& z\ge 0\\\pi+2\tan^{-1} z& z<0\end{array}\right.
\end{align*}
where $$z=\frac{|(u\times v)\cdot w|}{1+u\cdot v+v\cdot w+w\cdot u}.$$
(Note that the denominator on the right-hand side above can be 0.) Training and test data were randomly sampled from the unit sphere. The neural network architecture and hyper parameters are specified above. When we use ReLU activation function,  the best MAE is 0.108. Replacing the activation function in the first layer by the Seagull activation function $\log(1+x^2)$, the best MAE is reduced to $0.080$. Increasing the training set to $50000$ data points, the MAEs reduce to $0.086$ for ReLU and $0.043$ for Seagull function, respectively; a single use of an even activation function leads to a $50\%$ improvement.

In real-world applications, data always contains noise. To test the effectiveness of our method on noisy data, we added random noise of mean $0$ and $5\%$ standard deviation onto the training labels. The results in table \ref{tab:noise} demonstrate the robustness and effectiveness of the even activation function.

\begin{table*}[t]
\centering
\begin{tabular}{cccccc}
\hline
  & ReLU & ELU & Sigmoid & Tanh & SoftPlus\\
\hline
$f(x)$ & 0.126 \textbf{(0.054)} & 0.078 \textbf{(0.043)} & 0.159 \textbf{(0.035)}  & 0.236 \textbf{(0.081)}  & 0.059 \textbf{(0.032)}\\
 \hline
$\log(1+f(x))$ & 0.041 \textbf{(0.020)} & 0.027 \textbf{(0.018)} & 0.042 \textbf{(0.013)}   & 0.087 \textbf{(0.025)}& 0.022 \textbf{(0.012)} \\ \hline
 $e^{f(x)}/100$  & 0.160 \textbf{(0.106)} & 0.123 \textbf{(0.090)} & 0.185 \textbf{(0.054)}  & 0.266 \textbf{(0.136)} & 0.094 \textbf{(0.056)}                                 \\ \hline
 $\sin(f(x))$ & 0.092 \textbf{(0.034)} & 0.056 \textbf{(0.027)}& 0.271 \textbf{(0.017)}  & 0.174 \textbf{(0.040)} & 0.032 \textbf{(0.021)}         \\ \hline
$\sqrt{\frac{f^2(x)+3}{f(x)+1}}$ & 0.026 \textbf{(0.010)} & 0.018 \textbf{(0.012)}  & 0.038 \textbf{(0.011)}   & 0.054 \textbf{(0.017)}  & 0.013 \textbf{(0.007)}                                \\ \hline
\end{tabular}
\caption{Performance of different activation functions on the regression task with added noise. Outside parentheses: MAE for the original model; Inside parentheses: MAE of the model when the first activation function is replaced by Seagull. Each experiment was performed 5 times independently. The standard deviations are small and insignificant to present. \label{tab:noise} }
\end{table*}

While it is possible that a more carefully designed neural network with longer training and finer tuning of the learning rates may produce predictions with smaller MAE errors, the effect of using an even activation function in one of the layers is evident from Example 1.

{\bf Example 2}
The results of Example 1 demonstrated that for low-dimensional (9-dimensional) partially exchangeable target functions, the use of an even activation function improved network performance. To further evaluate its effectiveness on high-dimensional partially exchangeable target functions, we performed experiments on CIFAR 10, a popular data set that consists of $60000$ $32\times32$ color images in $10$ classes, with $6000$ images per class.
More details about the CIFAR-10 data set could be found in \cite{krizhevsky2009learning}

The following DCNNs were tested: Regnet \cite{radosavovic2020designing}, Resnet \cite{he2016deep}, VGG \cite{simonyan2014very}, and DPN \cite{chen2017dual}.
The details of each of these networks vary. In particular, we employ Regnet with 200 million flops (Regnet200MF), 50-layer Resnet (Resnet50), VGG net with 13 convolutional layers and 3 fully connected layers (VGG16), and DPN-26.
There are different ways to introduce an even activation function into the DCNNs.
Specifically, we replaced the ReLU at the output of each block for Regnet200MF, Resnet50, and VGG16.
Taking Resnet50 as an example, each block in the Resnet50 consists of a few weighted layers and a skip connection.
We replaced the last ReLU, which directly affects the output of the block, with Seagull activation function $\log(1+x^2)$ and left the rest of the activation functions unchanged.

For training setup, 50,000 images were used as training samples and 10,000 images for testing.
The total epoch number was limited as 350.
We selected the SGD optimizer with 0.9 momentum and set the learning rate as 0.1 for epoch $[0, 150)$, $0.01$ for epoch $[150, 250)$, and $0.001$ for epoch $[250, 350)$.
All training was started from scratch.
Considering the random initialization of the DCNNs, each training process was performed 5 times, and the average accuracy on testing data and standard deviation were summarized in Figure \ref{fig:SeagullvsReLU}.

\begin{figure}[t]
\centering
\includegraphics[width=1.05 \columnwidth]{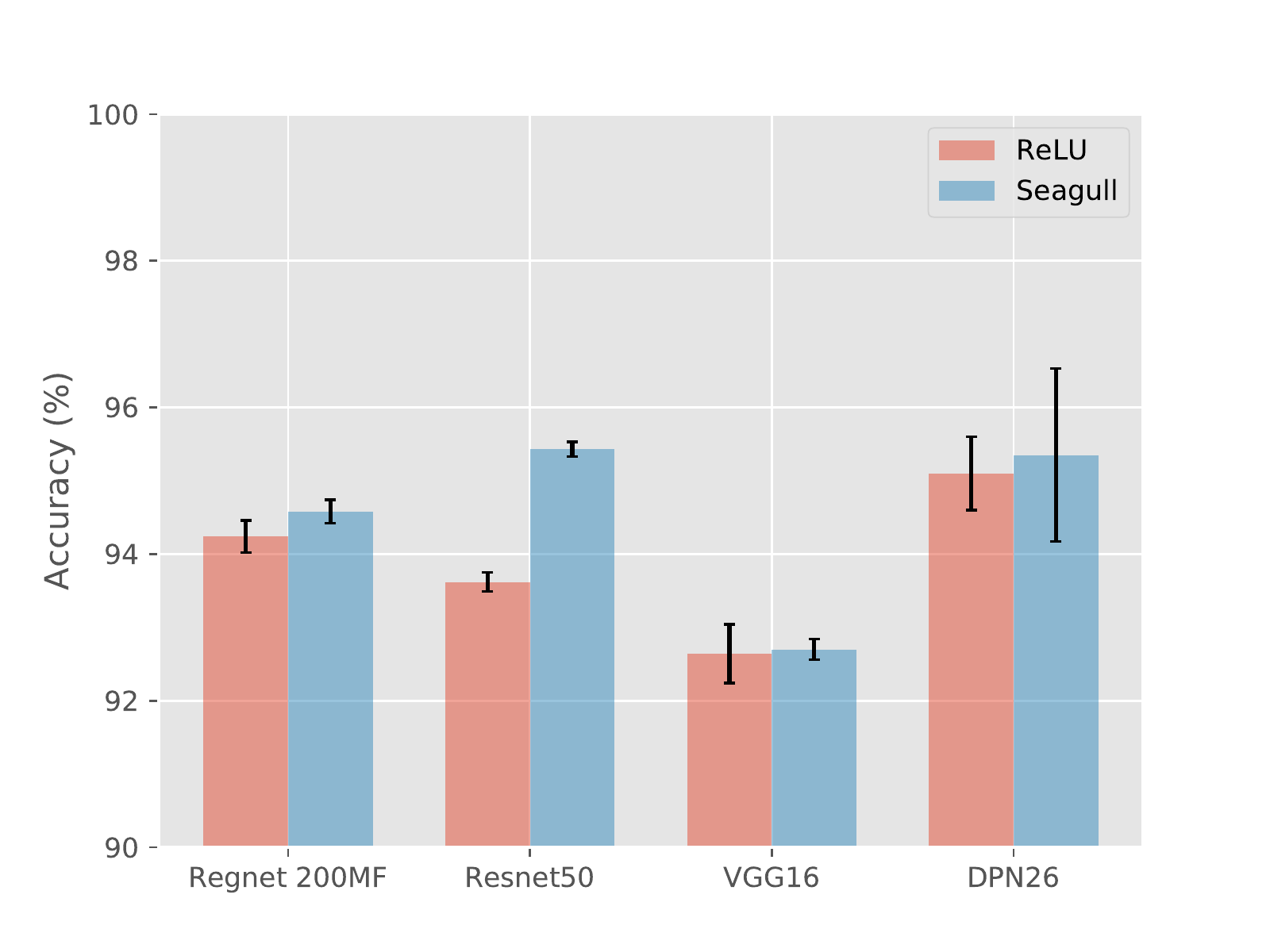} 
\caption{Performance of Seagull activation function and ReLU on CIFAR-10 under different neural network structures. The red bar reflect the average accuracy of ReLU and the blue bar for the Even activation function on CIFAR-10 data set, respectively. The black lines on the top of the bars present the standard deviation of the accuracy.}
\label{fig:SeagullvsReLU}
\end{figure}

The results in Figure \ref{fig:SeagullvsReLU} reveal some interesting information.
The Seagull function presents significant improvement compare to the ReLU function on all three DCNNs, despite some variation.
The feature maps in the network retain the partial exchangeability of the images.
Considering the DCNN structures, the convolution operation and activation functions leave the spatial relationship unchanged as well as the partial exchangeability.
According to Theorem \ref{th1}, the even activation function could capture this feature better compared to ReLU.
While one might argue the rotation and flip invariance could be achieved by data augmentation and filter rotation \cite{gao2017efficient}, using the proposed Seagull function would provide the same feature, as well as significant reduction of the network size and training time.

For DPN-26, a different strategy was selected.
First, we train 300 epochs on the original DPN-26 from scratch. Then we insert a fully-connected layer with an even activation function and trained the neural network with the same hyper parameters for 300 epochs.
The result shows noticeable improvements: For the network without using even activation functions, the average accuracy reached $95.10\%\pm 0.11\%$, whereas the network using an even activation function reached $95.35\% \pm 0.04\%$.
The $95.35\%$ accuracy also outperformed DPN-92, which consists of many more parameters.
This substantial improvement further demonstrates the effectiveness of using even activation functions for partially exchangeable target functions.

\section{Generalization}
The partially exchangeable assumption is just one of the many common algebraic assumptions one can make on the target functions. For example, one may consider the case when the target function satisfies partially anti-exchangeable assumption, i.e., $f(u,v,w)=-f(v,u,w)$. Such functions are also common. For example, if $f$ is the determinant of a matrix with $u$ and $u$ as its two row vectors, $f$ is partially anti-exchangeable. In such a case, it is not difficult to see from the proof of Theorem 1 that the function $g(a(\cdot))$ is odd. One may achieve this goal by using odd activation functions in all the layers. We will not discuss this case further because odd activation functions have been widely investigated.

Another common yet challenging case for which activation functions can play a key role is when the target function satisfies $f(tu,v)=t^\alpha f(u,v)$ for some $\alpha>0$. The solution to such a case requires very different techniques and is beyond the scope of this paper.

In the same spirit, but using different technical tools, one can also study the case where not algebraic relations but rather analytic properties are assumed. For example, one may have prior knowledge or a hypothesis that the unknown function is not only of bounded Lipschitz, but also of bounded mixed derivatives. In this case, the careful design of activation functions is also of great potential. For example, by local convexity analysis, we can glue two existing activation functions to further improve the neural network performance on all the tasks in Example 1. The work will be published elsewhere.

\section{Conclusion}
In this paper, we emphasized the importance of studying activation functions in neural networks.
We theoretically proved and experimentally validated on synthetic and real-world data sets that when the target function is partially exchangeable, adding an even activation function into the network structure can significantly improve neural network performance.
Through a special yet commonly encountered case, these results demonstrate the great potential of investigating activation functions.

\section{Discussion}
It is not a surprise that the designated ``Seagull" activation function outperformed ReLU on the CIFAR-10 data set when applied on the fully connected layer because the image classification is partially exchangeable, and the Seagull activation function is an even function. However, when Seagull is used on the intermediate layers where the neurons are not fully connected yet, the neural network performance on CIFAR-10 data set also improves. It seems that partial exchangeability exists globally and locally in a certain sense. An alternative explanation is that the even activation function could effectively capture useful information from the features that otherwise are cut off by ReLU.

\bibliography{ref.bib}

\end{document}